\documentclass[letterpaper, 10 pt, conference]{ieeeconf}
\IEEEoverridecommandlockouts    

\usepackage{amsmath}
\usepackage{amssymb}
\usepackage{float}
\usepackage{graphicx}
\usepackage{subfigure}
\usepackage{hyperref}
\usepackage{cleveref}
\usepackage{color}
\newcommand{\sinc}{\text{\rm sinc}}

\newtheorem{theorem}{Theorem}
\newtheorem{lemma}{Lemma}

\title{Globally-Attractive Logarithmic Geometric Control of a Quadrotor for Aggressive Trajectory Tracking}

\author{Jacob Johnson$^{1}$, Randal Beard$^{2}$
    \thanks{$^{1}$Graduate research assistant, Electrical and Computer Engineering, Brigham Young University, {\tt\small jjohns99@byu.edu}}%
    \thanks{$^{2}$Professor of Electrical and Computer Engineering, Brigham Young University, {\tt\small beard@byu.edu}}%
    \thanks{$^{*}$This work has been funded by the Center for Unmanned Aircraft Systems (C-UAS), a National Science Foundation Industry/University Cooperative Research Center (I/UCRC) under NSF award No. IIP-1650547, along with significant contributions from C-UAS industry members.}%
}

\begin{document}

\maketitle

\begin{abstract}
    We present a new quadrotor geometric control scheme that is capable of tracking highly aggressive trajectories. Our geometric controller uses the logarithmic map of SO(3) to express rotational error in the Lie algebra, and we show that it is globally attractive without requiring a complicated hybrid switching scheme. We show the performance of our controller against highly aggressive trajectories in simulation experiments. Additionally, we present an adaptation of this controller that allows us to interface effectively with the angular rate controllers on an onboard flight control unit and show the ability of this adapted control scheme to track aggressive trajectories on a quadrotor hardware platform.
\end{abstract}

\section{Introduction}

A large number of quadrotor control methods have been presented in the literature. These methods can be sorted into three general categories: those that are linear, those that are nonlinear and non-geometric, and those that are geometric. Linear control methods neglect or approximate the nonlinear dynamics of the quadrotor by linearizing about an equilibrium point and treating the resulting dynamics as if they were the true dynamics of the system. These methods perform well as long as the state of the system remains near the equilibrium, but fail when the state leaves the resulting region of attraction. 

Nonlinear non-geometric control methods~\cite{Bouabdallah2004} compensate for certain nonlinearities in the dynamics of the quadrotor and have large regions of attraction. However, they neglect the fact that the rotation states of the quadrotor belong to the special orthogonal group SO(3). These control methods usually approximate the rotation states as a vector of Euler angles, resulting in poor performance when the rotation of the quadrotor approaches the associated singularities, or as a unit quaternion, resulting in possible unwinding phenomena \cite{Bhat2000}.

Geometric control methods correctly model the rotation states on SO(3) and are derived using methods from differential geometry. They do not have singularities and avoid the unwinding phenomena. The performance of a geometric controller is dependent on the choice of error representation used. The popular controller of \cite{Lee2010} uses the Frobenius norm of the difference between the identity matrix and the error rotation matrix as a Lyapunov function, which results in an error representation that performs poorly when the error is high (i.e. near 180 degrees). This issue was acknowledged in \cite{Lee2012} and a new error representation was proposed. However, this representation seems to lack motivation from the physics or dynamics of the system. Another representation uses the logarithmic map of SO(3) to express error in the Lie algebra $\mathfrak{so}(3)$~\cite{Bullo1995}. The logarithm maps geodesics, or shortest paths, on SO(3) to straight lines in $\mathfrak{so}(3)$, so we believe this is the most natural way to express rotational error. The logarithmic error representation has been used in several prior quadrotor controllers~\cite{Shi2017}~\cite{Yu2015}.

\begin{figure} 
    \centering
    \includegraphics[width=0.3\textwidth]{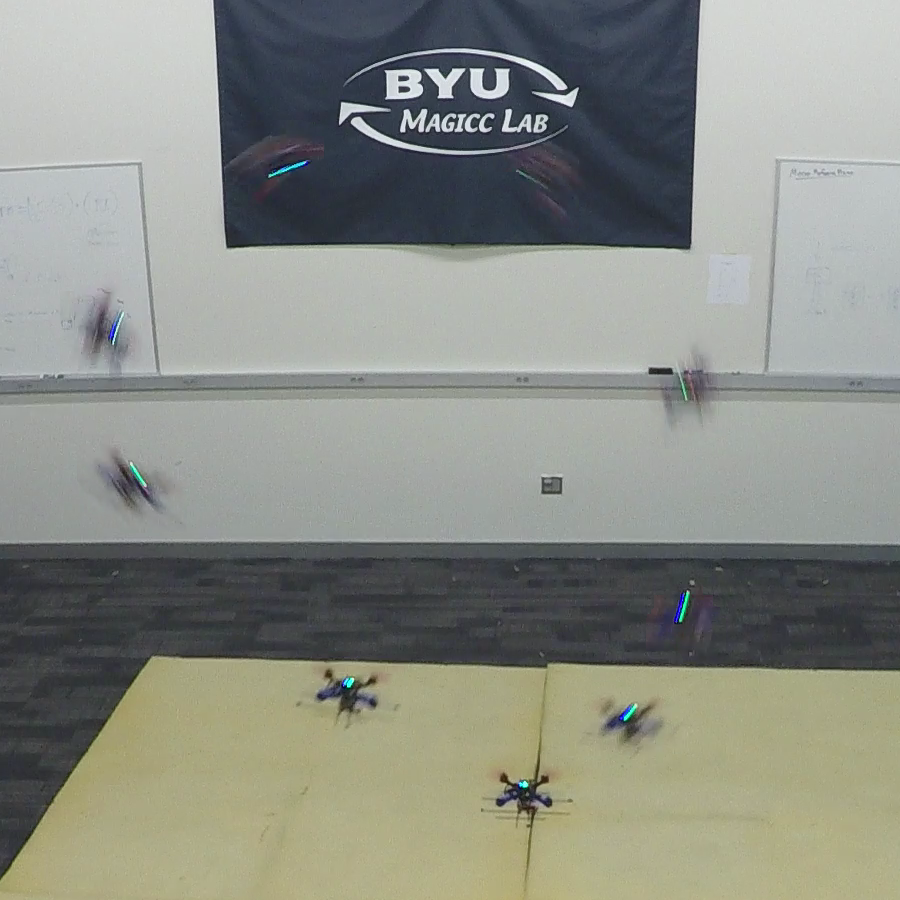}
    \caption{Time lapse image of our proposed geometric control scheme tracking a flipping loop trajectory.}
    \label{fig:flip_timelapse}
\end{figure}

It is well-known that no smooth control law on a compact manifold can be globally stable~\cite{Bhat2000}. For this reason the above-mentioned controllers are almost globally stable, i.e., there is a set of initial rotations (where the body rotation is exactly 180 degrees from the desired rotation) that are not in the region of attraction. Several non-smooth hybrid control schemes have been proposed to address this issue~\cite{Lee2015}~\cite{Yu2016}, and global stability is proven. However, hybrid controllers have complicated implementations, and the resulting jump dynamics often introduce undesirable non-smooth responses.

In this paper, we present a new geometric controller and implement it in a full trajectory tracking quadrotor control scheme. Our controller uses the logarithmic map to express the rotation error in $\mathfrak{so}(3)$, and we show that this controller is globally asymptotically stable. Our controller is discontinuous on the set where the rotation error is exactly 180 degrees, and so the results of~\cite{Bhat2000} do not apply, but it is not hybrid, so it has a more simple implementation and more smooth dynamic response than e.g.~\cite{Lee2015}. Our controller is similar to the one presented in \cite{Yu2016}, where the stronger condition of global exponential stability is proven. However, our formulation does not require describing the dynamics as a hybrid system on $\mathfrak{so}(3)$, and our proof is more simple. Additionally, we develop an adaptation to our controller that allows it to interface with the angular rate controllers that run at very high frequency on off-the-shelf onboard flight control units (FCUs), and we present highly aggressive trajectory tracking results on a hardware platform.

\section{Preliminaries}

We use the vector notation $\mathbf{t}_{a/b}^c \in \mathbb{R}^3$ to denote a value $\mathbf{t}$ (e.g. position, velocity) of coordinate frame $a$ with respect to frame $b$ expressed in frame $c$. Thus $\mathbf{t}_{a/b}^c = -\mathbf{t}_{b/a}^c$ and $\mathbf{t}_{a/b}^a = \mathbf{R}_c^a \mathbf{t}_{a/b}^c$, where $\mathbf{R}_c^a$ is a rotation matrix that re-expresses vectors from frame $c$ to frame $a$. The set of all 3D rotation matrices is isomorphic to the special orthogonal group, which can therefore be expressed as
\begin{equation}
    \text{SO}(3) = \left\{ \mathbf{R} \in \mathbb{R}^{3 \times 3} \: | \: \mathbf{R}^\top \mathbf{R} = \mathbf{I}, \: \text{det}(\mathbf{R}) = 1\right\},
\end{equation}
equipped with the group action of matrix multiplication. This set satisfies the group axioms and forms a smooth manifold, making SO(3) a Lie group. The Lie algebra of SO(3) (denoted $\mathfrak{so}(3)$) is the set of $3 \times 3$ skew-symmetric matrices and is isomorphic to $\mathbb{R}^3$ under the hat map
\begin{equation}
    \boldsymbol{\phi}^\wedge = \begin{bmatrix} 0 & -\phi_3 & \phi_2 \\ \phi_3 & 0 & -\phi_1 \\ -\phi_2 & \phi_1 & 0 \end{bmatrix}
\end{equation}
for $\boldsymbol{\phi} \in \mathbb{R}^3$. Skew-symmetric matrices can be mapped back to $\mathbb{R}^3$ using the vee map $\left( \boldsymbol{\phi}^\wedge \right)^\vee = \boldsymbol{\phi}$. For $\mathbf{R} \in \text{SO}(3)$,
\begin{equation}
    \left( \mathbf{R} \boldsymbol{\phi} \right)^\wedge = \mathbf{R} \boldsymbol{\phi}^\wedge \mathbf{R}^\top,
\end{equation}
and for $\mathbf{a}, \mathbf{b} \in \mathbb{R}^3$, $\mathbf{a}^\wedge \mathbf{b} = -\mathbf{b}^\wedge \mathbf{a}$.

The exponential map 
\begin{equation}
    \text{Exp}(\boldsymbol{\phi}) = \mathbf{I} + \text{sin}(\phi)\mathbf{u}^\wedge + (1 - \text{cos}(\phi)) \mathbf{u}^\wedge \mathbf{u}^\wedge,
    \label{eq:exp}
\end{equation}
where $\boldsymbol{\phi} = \phi \mathbf{u}$ and $\mathbf{u} \in \mathbb{R}^3$ is a unit vector, can be used to map from $\mathbb{R}^3$ to SO(3). Its inverse is the logarithmic map 
\begin{subequations}
\begin{align}
    \phi &= \text{cos}^{-1}\left(\frac{\text{tr}(\mathbf{R}) - 1}{2}\right),
    \label{eq:log} \\
    \text{Log}(\mathbf{R}) &\stackrel{\triangle}{=} \phi \mathbf{u}   
    = \frac{1}{2 \sinc(\phi/2)\cos(\phi/2)}(\mathbf{R} - \mathbf{R}^\top)^\vee, \label{eq:log1} 
\end{align}
\end{subequations}
where $\sinc(x)\stackrel{\triangle}{=}\sin(x)/x$ is nonzero for $x\in(-\pi, \pi)$ and so $\sinc(\phi/2)$ is nonzero for $\phi\in(-2\pi, 2\pi)$.  If $\phi=\pm\pi$ as computed from Equation~\eqref{eq:log} then Equation~\eqref{eq:log1} is not defined, however from Equation~\eqref{eq:exp} we have that $\mathbf{R}\mathbf{u}=\mathbf{u}$ and therefore $\mathbf{u}$ can be computed using an eigen-decomposition, implying that $\text{Log}(\mathbf{R})$ is well defined on SO(3).
Additionally, we will make use of the left Jacobian of SO(3)
\begin{equation}
\begin{split}
    J_l(\boldsymbol{\phi}) &= \mathbf{I} + \sin\left(\phi/2\right)\sinc\left(\phi/2\right)\mathbf{u}^\wedge + (1-\sinc(\phi)) \mathbf{u}^\wedge \mathbf{u}^\wedge \\
    &= \int_0^1 \text{Exp}(\boldsymbol{\phi})^\alpha d\alpha
    \label{eq:jl}
\end{split}
\end{equation}
and its inverse
\begin{equation}
    J_l^{-1}(\boldsymbol{\phi}) = \mathbf{I} - \frac{\phi}{2} \mathbf{u}^\wedge + \left( 1 - \frac{\cos\left(\phi/2\right)}{\sinc\left(\phi/2\right)} \right) \mathbf{u}^\wedge \mathbf{u}^\wedge,
    \label{eq:jl_inv}
\end{equation}
where we note that $J_l(\boldsymbol{\phi})$ and $J_l^{-1}(\boldsymbol{\phi})$ are well defined on $\phi\in[-\pi, \pi]$.

\section{Quadrotor Dynamics}

The state of the quadrotor is given by the tuple $\mathbf{x} = \left( \mathbf{p}_{b/i}^i, \: \mathbf{v}_{b/i}^i, \: \mathbf{R}_b^i, \: \boldsymbol{\omega}_{b/i}^b \right)$, where $\mathbf{p}_{b/i}^i, \: \mathbf{v}_{b/i}^i \in \mathbb{R}^3$ are the position and velocity of the body frame (the coordinate system whose origin lies at the center of mass of the vehicle, with the $\mathbf{i}$ and $\mathbf{j}$ axes pointing out the front and right sides of the vehicle and the $\mathbf{k}$ axis pointing out its underside) expressed in some north-east-down inertial frame, $\mathbf{R}_b^i \in \text{SO}(3)$ is the rotation from the body frame to the inertial frame, and $\boldsymbol{\omega}_{b/i}^b \in \mathbb{R}^3$ is the angular velocity of the body frame expressed in the body frame.

We model the dynamics of the quadrotor using the equations \cite{Mahony2012}
\begin{subequations}
\begin{align}
    \dot{\mathbf{p}}_{b/i}^i &= \mathbf{v}_{b/i}^i, \label{eq:p_dyn} \\
    \dot{\mathbf{v}}_{b/i}^i &= g\mathbf{e}_3 - \frac{T}{m} \mathbf{R}_b^i \mathbf{e}_3, \label{eq; v_dyn}\\
    \dot{\mathbf{R}}_b^i &= \mathbf{R}_b^i {\boldsymbol{\omega}_{b/i}^b}^\wedge, \label{eq:rot_dyn}\\
    \mathbf{J}\dot{\boldsymbol{\omega}}_{b/i}^b &= -{\boldsymbol{\omega}_{b/i}^b}^\wedge \mathbf{J} \boldsymbol{\omega}_{b/i}^b + \boldsymbol{\tau}^b,
    \label{eq:ang_dyn}
\end{align}
\label{eq:dynamics}%
\end{subequations}
where $g$ is the gravitational constant, $m$ is the mass of the vehicle, $\mathbf{J} \in \mathbb{R}^{3 \times 3}$ is the inertia matrix, $T$ is the total force produced by the rotors, $\boldsymbol{\tau}^b \in \mathbb{R}^3$ is the total moment vector produced by the rotors expressed in the body frame, and $\mathbf{e}_3 = \begin{bmatrix} 0 & 0 & 1 \end{bmatrix}^\top$. Motor throttles $\boldsymbol{\delta} \in \mathbb{R}^4$, $\delta_i \in [0,1]$ can be mapped to a total thrust and moment vector using the linear relationship
\begin{equation}
    \begin{bmatrix} T \\ \boldsymbol{\tau}^b \end{bmatrix} = \mathbf{M} \boldsymbol{\delta},
    \label{eq:mix}
\end{equation}
where $\mathbf{M} \in \mathbb{R}^{4 \times 4}$ is an invertible constant mixing matrix that captures vehicle-specific configuration details, such as the position of each rotor with respect to the center of mass, the amount of thrust and torque a single rotor is able to produce, etc. See \cite{Mahony2012} for more details.

\section{Controller Architecture}

\begin{figure}
    \centering
    \includegraphics[width=0.3\textwidth]{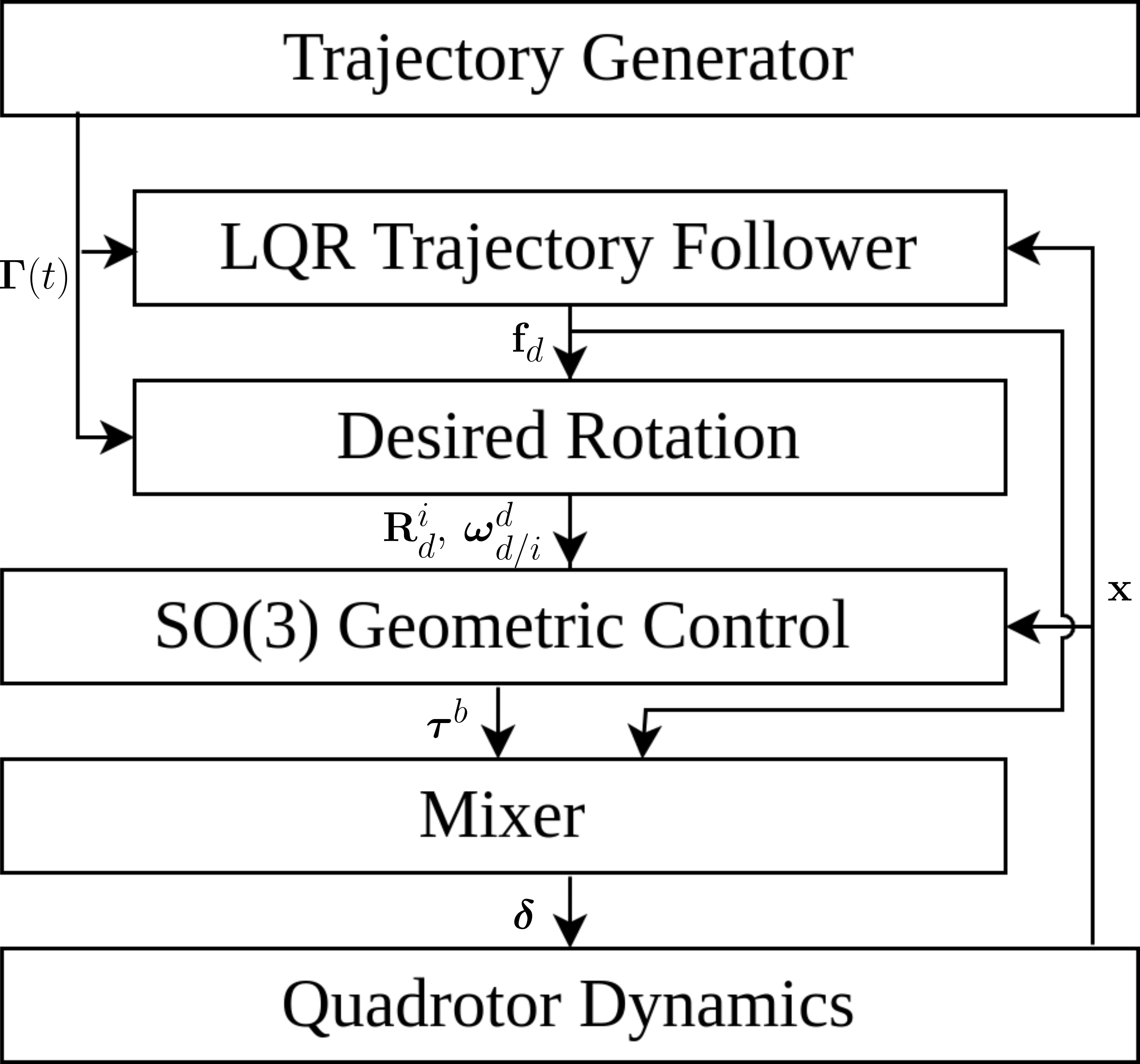}
    \caption{The architecture of the proposed control scheme.}
    \label{fig:arch}
\end{figure}

The architecture of the proposed controller is shown in Figure \ref{fig:arch}. We aim to follow three-times-differentiable trajectories, along with a desired heading and heading rate. The trajectory generator block provides the desired trajectory parameters at time $t$, represented by the tuple $\boldsymbol{\Gamma}(t) = \left(\mathbf{p}_d(t), \: \dot{\mathbf{p}}_d(t), \: \ddot{\mathbf{p}}_d(t), \dddot{\mathbf{p}}_d(t), \: \psi_d(t), \: \dot{\psi}_d(t)\right)$, where $\mathbf{p}_d(t), \: \dot{\mathbf{p}}_d(t), \: \ddot{\mathbf{p}}_d(t), \: \dddot{\mathbf{p}}_d(t) \in \mathbb{R}^3$ are respectively the desired position, velocity, acceleration and jerk of the body frame with respect to the inertial frame expressed in the inertial frame, and $\psi_d(t), \: \dot{\psi}_d(t) \in \mathbb{R}$ are the desired heading and heading rate. A trajectory following LQR controller uses the desired trajectory and the current state to produce a desired force vector $\mathbf{f}_d \in \mathbb{R}^3$. Using the fact that quadrotors are only capable of producing force along the body $\mathbf{k}$-axis, a desired rotation matrix $\mathbf{R}_d^i \in \text{SO}(3)$ and angular rate $\boldsymbol{\omega}_{d/i}^d$ are computed so that the desired $\mathbf{k}$-axis aligns with the desired force $\mathbf{f}_d$. A geometric controller on SO(3) uses these desired rotation states in addition to the current vehicle state to compute the torque the rotors must produce in order to drive the rotational error to zero. The required forces and torques are then mixed using the inverse of \eqref{eq:mix} to find the required motor throttles, which are then saturated between 0 and 1 before they are fed to the motors on the quadrotor.

Prior works~\cite{Lee2010}~\cite{Yu2016} prove the stability of the coupled translational and rotational dynamics given certain initial conditions. We will not do so in this paper given space limitations.

\section{Trajectory-following LQR Control}

We begin by assuming that the quadrotor is able to produce any desired force $\mathbf{f}_d$ (we will relax this assumption in the next section). The position and velocity dynamics then become 
\begin{subequations}
\begin{align}
    \dot{\mathbf{p}}_{b/i}^i &= \mathbf{v}_{b/i}^i, \\
    \dot{\mathbf{v}}_{b/i}^i &= g \mathbf{e}_3 + \frac{1}{m} \mathbf{f}_d.
\end{align}
\end{subequations}
Define the error states
\begin{subequations}
\begin{align}
    \mathbf{e}_p &= \mathbf{p}_{b/i}^i - \mathbf{p}_d, \\
    \mathbf{e}_v &= \mathbf{v}_{b/i}^i - \dot{\mathbf{p}}_d.
\end{align}
\end{subequations}
Taking the time derivative, we get the error dynamics
\begin{subequations}
\begin{align}
    \dot{\mathbf{e}}_p &= \mathbf{e}_v, \label{eq:ep_dot}\\
    \dot{\mathbf{e}}_v &= g \mathbf{e}_3 + \frac{1}{m} \mathbf{f}_d - \ddot{\mathbf{p}}_d.
    \label{eq:ev_dot}
\end{align}
\label{eq:error_dynamics}
\end{subequations}
Note that these dynamics are linear. Define $\tilde{\mathbf{f}} = \mathbf{f}_d - \mathbf{f}_{\text{eq}}$, where 
\begin{equation}
    \mathbf{f}_{\text{eq}} = m(-g \mathbf{e}_3 + \ddot{\mathbf{p}}_d)
\end{equation}
is the force at equilibrium. Then the error state dynamics \eqref{eq:error_dynamics} are represented in state-space form as 
\begin{equation}
    \dot{\mathbf{e}} = \begin{bmatrix} \mathbf{0} & \mathbf{I} \\ \mathbf{0} & \mathbf{0}\end{bmatrix} \mathbf{e} + \begin{bmatrix} \mathbf{0} \\ \frac{1}{m} \mathbf{I} \end{bmatrix} \tilde{\mathbf{f}},
\end{equation}
where $\mathbf{e} = \begin{bmatrix} \mathbf{e}_p^\top & \mathbf{e}_v^\top \end{bmatrix}^\top$. Additionally, we can augment these error dynamics with the integrator
\begin{equation}
    \mathbf{e}_i = \int_0^t \mathbf{e}_p \: dt
\end{equation}
such that
\begin{equation}
    \dot{\mathbf{e}}_a = \underbrace{\begin{bmatrix} \mathbf{0} & \mathbf{I} & \mathbf{0} \\ \mathbf{0} & \mathbf{0} & \mathbf{0} \\ \mathbf{I} & \mathbf{0} & \mathbf{0} \end{bmatrix}}_{\mathbf{A}_a} \mathbf{e}_a + \underbrace{\begin{bmatrix} \mathbf{0} \\ \frac{1}{m} \mathbf{I} \\ \mathbf{0} \end{bmatrix}}_{\mathbf{B}_a} \tilde{\mathbf{f}},
    \label{eq:e_aug}
\end{equation}
where $\mathbf{e}_a = \begin{bmatrix} \mathbf{e}_p^\top & \mathbf{e}_v^\top & \mathbf{e}_i^\top \end{bmatrix}^\top$. These dynamics can easily be shown to be controllable.

Choosing the LQR objective function
\begin{equation}
    \mathcal{J}_\text{LQR}(\mathbf{e}_a, \tilde{\mathbf{f}}) = \int_0^\infty \left( \mathbf{e}_a^\top \mathbf{W}_e \mathbf{e}_a + \tilde{\mathbf{f}}^\top \mathbf{W}_f \tilde{\mathbf{f}} \right) dt,
    \label{eq:lqr}
\end{equation}
where $\mathbf{W}_e \in \mathbb{R}^{9 \times 9}$ and $\mathbf{W}_f \in \mathbb{R}^{3 \times 3}$ are symmetric positive definite matrices, the controller that minimizes \eqref{eq:lqr} and exponentially stabilizes the error dynamics \eqref{eq:e_aug} is given by
\begin{equation}
    \tilde{\mathbf{f}} = -\mathbf{K} \mathbf{e}_a \: \Longrightarrow \: \mathbf{f}_d = -\mathbf{K} \mathbf{e}_a + \mathbf{f}_\text{eq},
\end{equation}
where $\mathbf{K} = \mathbf{W}_f^{-1} \mathbf{B}_a^\top \mathbf{P}_\text{LQR}$ and $\mathbf{P}_\text{LQR}$ is the solution to the continuous-time algebraic Riccati equation.

\section{Desired Rotation}

From the trajectory-following LQR controller we receive a desired force vector $\mathbf{f}_d$. In the previous section we assumed that the quadrotor could produce any desired force, but in reality it can only produce force in the direction of its rotors, along the body $\mathbf{k}$-axis. The vehicle will be able to achieve the desired force only if this axis is aligned with the force vector.

We follow the method presented in \cite{Lee2010} to construct a desired rotation matrix $\mathbf{R}_d^i \in \text{SO}(3)$ (the rotation from the ``desired" frame to the inertial frame) such that the desired $\mathbf{k}$-axis is aligned with $\mathbf{f}_d$. Noting that the columns of a rotation matrix are the coordinate-frame axis vectors, we set $\mathbf{R}_d^i = \begin{bmatrix} \mathbf{i}_d & \mathbf{j}_d & \mathbf{k}_d\end{bmatrix}$, where $\mathbf{i}_d, \; \mathbf{j}_d, \; \text{and } \mathbf{k}_d$ are the desired coordinate axes expressed in the inertial frame. Set 
\begin{subequations}
\begin{equation}
    \mathbf{k}_d = -\frac{\mathbf{f}_d}{\lVert \mathbf{f}_d \rVert}
    \label{eq:kd}
\end{equation}
to align the rotors to the desired force vector. The desired rotation about $\mathbf{k}_d$ can be chosen arbitrarily. To constrain the rotation matrix, we provide a desired heading $\psi_d$ from the trajectory generator, implying that 
\begin{equation}
    \mathbf{j}_d = \frac{\mathbf{k}_d \times \mathbf{s}_d}{\lVert \mathbf{k}_d \times \mathbf{s}_d \rVert}, 
    \qquad \mathbf{i}_d = \mathbf{j}_d \times \mathbf{k}_d,
    \label{eq:jd}
\end{equation}
where $\mathbf{s}_d = \begin{bmatrix} \text{cos}(\psi_d) & \text{sin}(\psi_d) & 0\end{bmatrix}^\top$.
\label{eq:desired_axes}
\end{subequations}

The desired angular velocity $\boldsymbol{\omega}_{d/i}^d$ is constructed from the rotational kinematics,
\begin{equation}
    \dot{\mathbf{R}}_d^i = \mathbf{R}_d^i {\boldsymbol{\omega}_{d/i}^d}^\wedge \: \Longrightarrow \: \boldsymbol{\omega}_{d/i}^d = \left( {\mathbf{R}_d^i}^\top \dot{\mathbf{R}}_d^i \right)^\vee,
\end{equation}
where $\dot{\mathbf{R}}_d^i = \begin{bmatrix} \dot{\mathbf{i}}_d & \dot{\mathbf{j}}_d & \dot{\mathbf{k}}_d \end{bmatrix}$, and  $\dot{\mathbf{i}}_d, \; \dot{\mathbf{j}}_d, \; \text{and } \dot{\mathbf{k}}_d$ are found by differentiating \eqref{eq:desired_axes}.

Contrary to what is done in \cite{Lee2010}, we set the total thrust of the motors to be 
\begin{equation}
    T = \lVert \mathbf{f}_d \rVert
\end{equation}
as opposed to $T = -\mathbf{f}_d^\top {\mathbf{R}_b^i}^\top \mathbf{e}_3$. While the latter can be proven to stabilize the full rigid-body dynamics when the attitude tracking error is within a bounded region~\cite{Lee2010}, we found that the former was able to track much more aggressive trajectories. This is likely because the thrust is not as limited, and the rotational system converges quickly enough that the direction of applied thrust is almost always close to the desired direction of thrust.

\section{Control on SO(3) Using Logarithmic Error}
\label{sec:geom_ctrl}

We develop a geometric controller on SO(3) to track the desired rotation $\mathbf{R}_d^i$ and angular velocity $\boldsymbol{\omega}_{d/i}^d$ given the dynamics (\ref{eq:rot_dyn},~\ref{eq:ang_dyn}) and prove that it is globally attractive. 

\begin{figure}
    \centering
    \includegraphics[width=0.47\textwidth]{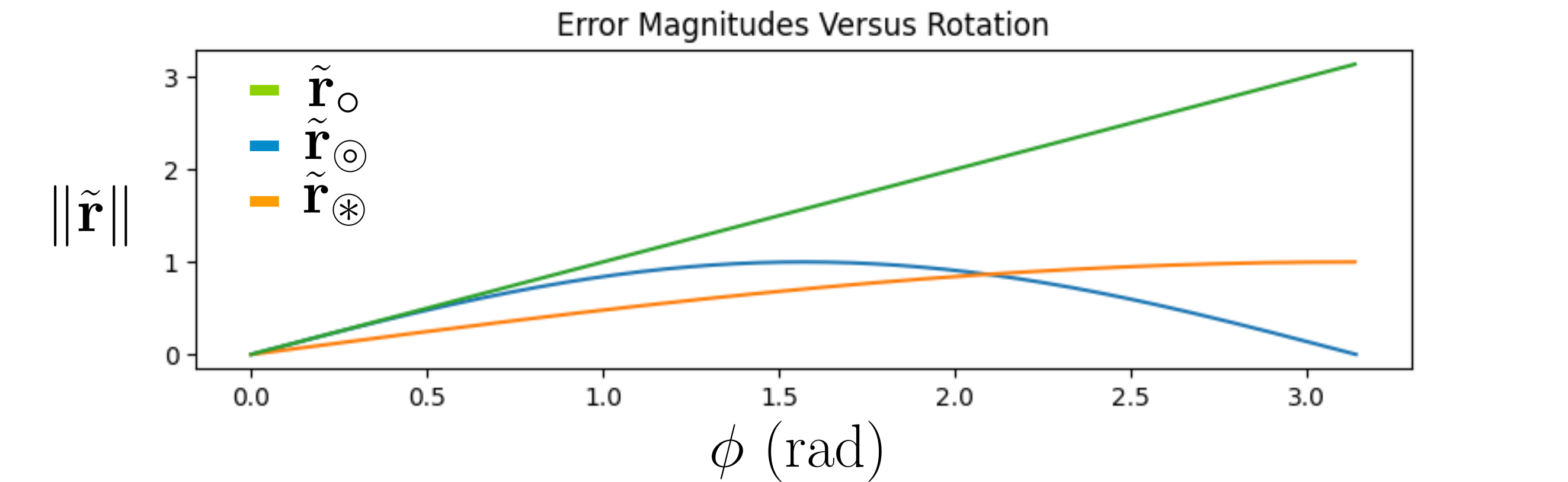}
    \caption{Error function comparisons.}
    \label{fig:err_func}
\end{figure}

The rotation from the desired frame to the body frame is $\mathbf{R}_d^b = {\mathbf{R}_b^i}^\top \mathbf{R}_d^i$. We evaluate three choices for the error rotation: $\tilde{\mathbf{r}}_\circledcirc = \frac{1}{2}(\mathbf{R}_d^b - {\mathbf{R}_d^b}^\top)^\vee$ which comes from using the Lyapunov function $\frac{1}{4} \left\|\mathbf{I} - \mathbf{R}_d^b\right\|_F^2$~\cite{Lee2010}, $\tilde{\mathbf{r}}_\circledast = \frac{1}{2 \sqrt{1 + \text{tr}(\mathbf{R}_d^b})}(\mathbf{R}_d^b - {\mathbf{R}_d^b}^\top)^\vee$ used in \cite{Lee2012}, and the logarithmic map $\tilde{\mathbf{r}}_\circ = \text{Log}(\mathbf{R}_d^b)$. Figure \ref{fig:err_func} shows the magnitude of each error expression versus $\phi \in [0,\pi]$ given that $\mathbf{R}_d^b = \text{Exp}(\begin{bmatrix} \phi & 0 & 0\end{bmatrix}^\top)$.
Unsurprisingly, $\lVert\tilde{\mathbf{r}}_\circ\rVert$ increases linearly with $\phi$ because $\lVert \tilde{\mathbf{r}}_\circ \rVert = \lVert \text{Log}(\mathbf{R}_d^b) \rVert = \phi$. This shows that the logarithmic map maps geodesics in $\text{SO}(3)$ to straight lines in the Lie algebra. For this reason we believe that the logarithmic map is the most effective and natural method for representing rotational error, thus we choose the rotational error expression to be 
\begin{equation}
    \tilde{\mathbf{r}} = \tilde{\mathbf{r}}_\circ = \text{Log}(\mathbf{R}_d^b).
    \label{eq:r_err}
\end{equation}
We express the error in angular velocity as
\begin{equation}
    \tilde{\boldsymbol{\omega}} = \mathbf{R}_d^b \boldsymbol{\omega}_{d/i}^d - \boldsymbol{\omega}_{b/i}^b = \boldsymbol{\omega}_{d/b}^b.
    \label{eq:om_err}
\end{equation}

\begin{lemma}
    The dynamics of $\tilde{\mathbf{r}}$ are given by 
    \begin{equation}
        \dot{\tilde{\mathbf{r}}} = J_l(\tilde{\mathbf{r}})^{-1} \tilde{\boldsymbol{\omega}},
        \label{eq:r_err_dot}
    \end{equation}
    where $J_l(\tilde{\mathbf{r}})$ is the left Jacobian of $\text{SO(3)}$.
\end{lemma}
\begin{proof}
    We follow the proof given in \cite{Barfoot2017}. Noting that $\mathbf{R}_d^b = \text{Exp}(\tilde{\mathbf{r}})$, we take the time derivative to obtain
    \begin{equation}
        \dot{\mathbf{R}}_d^b = \frac{d}{dt} \text{exp}(\tilde{\mathbf{r}}^\wedge) = \int_0^1 \text{exp}(\alpha \tilde{\mathbf{r}}^\wedge) \dot{\tilde{\mathbf{r}}}^\wedge \text{exp}((1-\alpha)\tilde{\mathbf{r}}^\wedge) d\alpha,
    \end{equation}
    where we have used the expression for the time derivative of the matrix exponential. Rearranging, we get
    \begin{equation}
    \begin{split}
        &\dot{\mathbf{R}}_d^b = \int_0^1 {\mathbf{R}_d^b}^\alpha \dot{\tilde{\mathbf{r}}}^\wedge {\mathbf{R}_d^b}^{1-\alpha} d\alpha = \left( \int_0^1 {\mathbf{R}_d^b}^\alpha \dot{\tilde{\mathbf{r}}}^\wedge {\mathbf{R}_d^b}^{-\alpha} d\alpha \right) \mathbf{R}_d^b \\
        &\Longrightarrow \dot{\mathbf{R}}_d^b {\mathbf{R}_d^b}^\top = \int_0^1 \left( {\mathbf{R}_d^b}^\alpha \dot{\tilde{\mathbf{r}}}\right)^\wedge d\alpha = \left( \int_0^1 {\mathbf{R}_d^b}^\alpha d\alpha \; \dot{\tilde{\mathbf{r}}} \right)^\wedge \\
        &= \left( J_l(\tilde{\mathbf{r}}) \dot{\tilde{\mathbf{r}}} \right)^\wedge.
    \end{split}
    \end{equation}
    Finally, noting that 
\begin{equation}
    \dot{\mathbf{R}}_d^b = \mathbf{R}_d^b {\boldsymbol{\omega}_{d/b}^d}^\wedge = \tilde{\boldsymbol{\omega}}^\wedge \mathbf{R}_d^b \: \Longrightarrow \: \tilde{\boldsymbol{\omega}} = \left( \dot{\mathbf{R}}_d^b {\mathbf{R}_d^b}^\top \right)^\vee,
\end{equation}
we get
\begin{equation}
    J_l(\tilde{\mathbf{r}}) \dot{\tilde{\mathbf{r}}} = \tilde{\boldsymbol{\omega}} \: \Longrightarrow \: \dot{\tilde{\mathbf{r}}} = J_l(\tilde{\mathbf{r}})^{-1} \tilde{\boldsymbol{\omega}}.
\end{equation}
\end{proof}

Additionally, the angular velocity error dynamics are given by 
\begin{equation}
    \mathbf{J} \dot{\tilde{\boldsymbol{\omega}}} = \mathbf{J}\dot{\boldsymbol{\omega}}_{d/i}^b + {\boldsymbol{\omega}_{b/i}^b}^\wedge \mathbf{J}\boldsymbol{\omega}_{b/i}^b - \boldsymbol{\tau}^b,
    \label{eq:om_err_dot}
\end{equation}
where we note that
\begin{equation}
\begin{split}
    \dot{\boldsymbol{\omega}}_{d/i}^b &= \frac{d}{dt}\left( \mathbf{R}_d^b \boldsymbol{\omega}_{d/i}^d \right) = \mathbf{R}_d^b \dot{\boldsymbol{\omega}}_{d/i}^d + \dot{\mathbf{R}}_d^b \boldsymbol{\omega}_{d/i}^d \\
    &= \mathbf{R}_d^b \dot{\boldsymbol{\omega}}_{d/i}^d + {\tilde{\boldsymbol{\omega}}}^\wedge \mathbf{R}_d^b \boldsymbol{\omega}_{d/i}^d \\
    &= \mathbf{R}_d^b \dot{\boldsymbol{\omega}}_{d/i}^d + (\mathbf{R}_d^b \boldsymbol{\omega}_{d/i}^d - \boldsymbol{\omega}_{b/i}^b)^\wedge \mathbf{R}_d^b \boldsymbol{\omega}_{d/i}^d \\
    &= \mathbf{R}_d^b \dot{\boldsymbol{\omega}}_{d/i}^d - {\boldsymbol{\omega}_{b/i}^b}^\wedge \mathbf{R}_d^b \boldsymbol{\omega}_{d/i}^d.
\end{split}
\end{equation}

Define the set 
\begin{equation}
    S \triangleq \{ \phi \mathbf{u} \: | \: -\pi < \phi < \pi, \: \mathbf{u}^\top \mathbf{u} = 1\},
    \label{eq:set_def}
\end{equation}
and its closure $\bar{S}$ where $-\pi\leq \phi \leq \pi$.  

\begin{theorem}
    \label{thm:full_dyn_ctrl}
    Given the dynamics~\eqref{eq:r_err_dot}, and~\eqref{eq:om_err_dot}, the control law
    \begin{equation}
        \boldsymbol{\tau}^b = {\boldsymbol{\omega}_{b/i}^b}^\wedge \mathbf{J} \boldsymbol{\omega}_{b/i}^b + \mathbf{J} \dot{\boldsymbol{\omega}}_{d/i}^b + J_l(\tilde{\mathbf{r}})^{-\top} \mathbf{K}_r \tilde{\mathbf{r}} + \mathbf{K}_\omega \tilde{\boldsymbol{\omega}},
        \label{eq:tau}
    \end{equation}
    where $\mathbf{K}_r, \: \mathbf{K}_\omega \in \mathbb{R}^{3 \times 3}$ are symmetric positive definite matrices, 
    is asymptotically stable for all
    $(\tilde{\mathbf{r}}, \tilde{\boldsymbol{\omega}})\in S\times \mathbb{R}^3$.  
    Furthermore, if $\mathbf{K}_r = k_r \mathbf{I}$ and $\mathbf{K}_\omega = k_\omega \mathbf{J}$ where $k_r$ and $k_\omega$ are scalars, then the closed-loop system is globally attractive on $\bar{S}\times\mathbb{R}^3$.
\end{theorem}
\begin{proof}
    Let $(\tilde{\mathbf{r}}, \tilde{\boldsymbol{\omega}})\in S\times \mathbb{R}^3$, and 
    consider the Lyapunov function candidate
    \begin{equation}
        \mathcal{V}(\tilde{\mathbf{r}}, \tilde{\boldsymbol{\omega}}) = \frac{1}{2} \tilde{\mathbf{r}}^\top \mathbf{K}_r \tilde{\mathbf{r}} + \frac{1}{2} \tilde{\boldsymbol{\omega}}^\top \mathbf{J} \tilde{\boldsymbol{\omega}}.
        \label{eq:lyap}
    \end{equation}
    Taking the time derivative, we get
    \begin{equation}
    \begin{split}
        \dot{\mathcal{V}} &= \dot{\tilde{\mathbf{r}}}^\top \mathbf{K}_r \tilde{\mathbf{r}} + \tilde{\boldsymbol{\omega}}^\top \mathbf{J} \dot{\tilde{\boldsymbol{\omega}}} \\ 
        &= \tilde{\boldsymbol{\omega}}^\top J_l(\tilde{\mathbf{r}})^{-\top} \mathbf{K}_r \tilde{\mathbf{r}} + \tilde{\boldsymbol{\omega}}^\top \left( \mathbf{J}\dot{\boldsymbol{\omega}}_{d/i}^b + {\boldsymbol{\omega}_{b/i}^b}^\wedge \mathbf{J}\boldsymbol{\omega}_{b/i}^b - \boldsymbol{\tau}^b \right) \\
        &= \tilde{\boldsymbol{\omega}}^\top \left( J_l(\tilde{\mathbf{r}})^{-\top} \mathbf{K}_r \tilde{\mathbf{r}} + \mathbf{J}\dot{\boldsymbol{\omega}}_{d/i}^b + {\boldsymbol{\omega}_{b/i}^b}^\wedge \mathbf{J}\boldsymbol{\omega}_{b/i}^b - \boldsymbol{\tau}^b \right)\\
        &= -\tilde{\boldsymbol{\omega}}^\top \mathbf{K}_\omega \tilde{\boldsymbol{\omega}},
        \label{eq:v_dot}
    \end{split}
    \end{equation}
    which is negative semi-definite. However, note that
    \begin{equation}
    \begin{split}
        \dot{\mathcal{V}} = 0 \: &\Longrightarrow \: \tilde{\boldsymbol{\omega}} \equiv 0 \Longrightarrow \: \dot{\tilde{\boldsymbol{\omega}}} \equiv 0 \\ 
        &\Longrightarrow \: \mathbf{J}\dot{\boldsymbol{\omega}}_{d/i}^b + {\boldsymbol{\omega}_{b/i}^b}^\wedge \mathbf{J}\boldsymbol{\omega}_{b/i}^b - \boldsymbol{\tau}^b \equiv 0 \\
        &\Longrightarrow \: J_l(\tilde{\mathbf{r}})^{-\top} \mathbf{K}_r \tilde{\mathbf{r}} = 0 
        \Longrightarrow \: \tilde{\mathbf{r}} = 0,
    \end{split}
    \end{equation}
    where the last result is due to the fact that the matrix $J_l(\tilde{\mathbf{r}})^{-\top} \mathbf{K}_r$ is full-rank, thereby showing that the largest invariant set in $S\times\mathbb{R}^3$ is the origin, and asymptotic stability follows by the LaSalle invariance principle.
    
    Now assume that $(\tilde{\mathbf{r}}(0), \tilde{\boldsymbol{\omega}}(0))\in (\bar{S}\setminus S) \times \mathbb{R}^3$, 
    where the initial body rotation $\mathbf{R}_b^i$ is exactly 180 degrees from the desired rotation $\mathbf{R}_d^i$,
    and assume that the set $(\bar{S}\setminus S) \times \mathbb{R}^3$ is invariant to the dynamics \eqref{eq:r_err_dot}, \eqref{eq:om_err_dot}, and~\eqref{eq:tau}. 
    
    Since $\tilde{\mathbf{r}}\in\bar{S}\setminus S$ we have that $\tilde{\mathbf{r}}=\pm \pi \mathbf{u}$, which implies that $\dot{\tilde{\mathbf{r}}} = \pm \pi \dot{\mathbf{u}}$.  Since $\mathbf{u}^\top\mathbf{u}=1$ we have that $\mathbf{u}^\top\dot{\mathbf{u}}=0$.  Therefore
\begin{equation}
    \mathbf{u}^\top\dot{\tilde{\mathbf{r}}} = \pm\pi\mathbf{u}^\top\dot{\mathbf{u}}=0.  
    \label{eq:u_rdot_zero}
\end{equation}
On the other hand, from Equation~\eqref{eq:r_err_dot} we have that
    \begin{align*}
        \mathbf{u}^\top \dot{\tilde{\mathbf{r}}} &= \mathbf{u}^\top J_{\ell}(\pm\pi\mathbf{u})^{-1}\tilde{\boldsymbol{\omega}} \\
        &= \mathbf{u}^\top \left(\mathbf{I} \mp\frac{\pi}{2}\mathbf{u}^\wedge + \mathbf{u}^\wedge\mathbf{u}^\wedge\right)\tilde{\boldsymbol{\omega}} \\
        &= \mathbf{u}^\top \tilde{\boldsymbol{\omega}} = 0,        
    \end{align*}
where we have used Equation~\eqref{eq:jl_inv} and the fact that $\mathbf{u}$ is orthogonal to $\mathbf{u}^\wedge\tilde{\boldsymbol{\omega}}$, and the last equality follows from Equation~\eqref{eq:u_rdot_zero}.
    Differentiating $\mathbf{u}^\top \tilde{\boldsymbol{\omega}} = 0$ gives
    \(
        \mathbf{u}^\top \dot{\tilde{\boldsymbol{\omega}}} = - \dot{\mathbf{u}}^\top \tilde{\boldsymbol{\omega}},
    \)
    where
    \begin{align*}
        \dot{\mathbf{u}}^\top \tilde{\boldsymbol{\omega}} 
        &= \pm\frac{1}{\pi}\tilde{\boldsymbol{\omega}}^\top J_l^{-\top}(\pm\pi \mathbf{u}) \tilde{\boldsymbol{\omega}} \\
        &= \pm\frac{1}{\pi}\tilde{\boldsymbol{\omega}}^\top (\mathbf{I} \mp \frac{\pi}{2}\mathbf{u}^\wedge + \mathbf{u}^\wedge \mathbf{u}^\wedge)\tilde{\boldsymbol{\omega}} \\
        &= \pm\frac{1}{\pi}\tilde{\boldsymbol{\omega}}^\top \mathbf{u} \mathbf{u}^\top \tilde{\boldsymbol{\omega}} = 0,
    \end{align*}
    where we have used the identity $\mathbf{u}^\wedge \mathbf{u}^\wedge = -\mathbf{I} + \mathbf{u} \mathbf{u}^\top$, and the fact that  $\tilde{\boldsymbol{\omega}}$ and $\mathbf{u}^\wedge \tilde{\boldsymbol{\omega}}$ are orthogonal. Therefore $\mathbf{u}^\top \dot{\tilde{\boldsymbol{\omega}}} = 0$ on $\bar{S}\setminus S$.
    
    Alternatively, from Equations~\eqref{eq:om_err_dot} and~\eqref{eq:tau} we get
    \begin{equation}
        \mathbf{u}^\top \dot{\tilde{\boldsymbol{\omega}}} = \mp\pi \mathbf{u}^\top \mathbf{J}^{-1} J_l^{-\top}(\pm\pi \mathbf{u}) \mathbf{K}_r \mathbf{u} - \mathbf{u}^\top \mathbf{J}^{-1}\mathbf{K}_\omega \tilde{\boldsymbol{\omega}}.
    \end{equation}
    If we let $\mathbf{K}_r = k_r \mathbf{I}$ and $\mathbf{K}_\omega = k_\omega \mathbf{J}$, then 
    \begin{equation}
        \mathbf{u}^\top \dot{\tilde{\boldsymbol{\omega}}} 
            = \mp\pi k_r \mathbf{u}^\top \mathbf{J}^{-1} \mathbf{u}.
    \end{equation}
    Since $\mathbf{J}^{-1}$ is positive definite and $\mathbf{u}$ is a unit vector, $\mathbf{u}^\top \dot{\tilde{\boldsymbol{\omega}}}\neq 0$, which is a contradiction.  Therefore $(\bar{S}\setminus S)\times \mathbb{R}^3$ is not invariant and the system dynamics must enter $S\times\mathbb{R}^3$, and thereby converge to the origin.  The closed-loop system is therefore globally asymptotically stable.
\end{proof}

In \eqref{eq:tau}, the inverse left Jacobian on the rotation error term is only necessary if $\mathbf{K}_r \neq k_r \mathbf{I}$. This control law is similar to the one presented in \cite{Yu2016}, where global exponential stability is proven. However, we feel that our proof is more simple in nature, without the need to express jumping dynamics in the Lie algebra of SO(3).

\section{Simulation Experiments}

We simulated the quadrotor dynamics \eqref{eq:dynamics} and tested the ability of the proposed control scheme to track highly aggressive trajectories. The dynamic parameters we used were $m = 1$ kg, $g = 9.81 \frac{\text{m}}{\text{s}^2}$, and $\mathbf{J} = \text{diag}(0.07, 0.07, 0.12)$ kg m$^2$. The dynamics and controller were updated synchronously at a frequency of 100 Hz. To create the mixing matrix $\mathbf{M}$, we gave the quadrotor an arm length of 0.25 m, a maximum thrust per rotor of 9.81 N, and a maximum torque per rotor of 5 Nm. Additionally, to demonstrate the robustness of the proposed control scheme we added zero-mean Gaussian input noise to each motor throttle input with a standard deviation of 0.04, and we perturbed the mixing matrix used in the controller by increasing the estimated thrust per rotor by 10 percent beyond its true value. The control parameters we used were $\mathbf{W}_e = \text{diag}(2.0, 2.0, 2.0, 1.0, 1.0, 1.0, 10^{-3}, 10^{-3}, 0.1)$, $\mathbf{W}_f = \text{diag}(0.1, 0.1, 1.0)$, $\mathbf{K}_r = 10 \mathbf{I}$, and $\mathbf{K}_\omega = 15 \mathbf{J}$, except where otherwise stated. A video of these experiments can be found at \url{www.youtube.com/watch?v=suEyw84wSoA}.

\subsection{Fast Circles}

We chose sinusoidal trajectories because of their $\mathcal{C}^\infty$ continuity and because they demonstrate the effectiveness of the proposed control scheme well. For the first trajectory, the quadrotor was commanded to follow circles in the $xy$-plane with a diameter of 10 m, a period of 2.5 s, a vertical offset of 5 m, and a commanded heading such that the body $i$-axis points in the direction of travel. 

\begin{figure}
    \centering
    \includegraphics[width=0.47\textwidth]{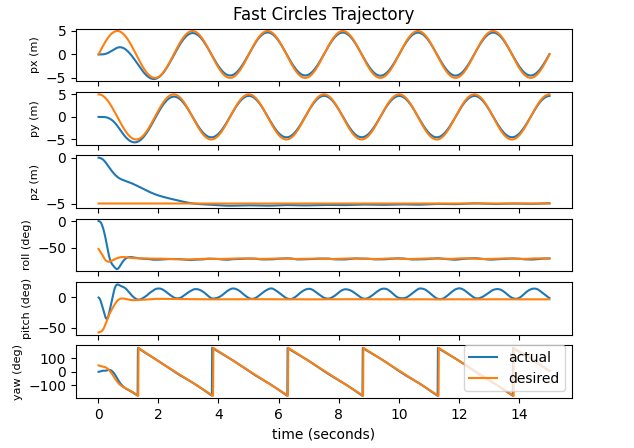}
    \caption{Fast circle trajectory performance.}
    \label{fig:fast_circles}
\end{figure}

\begin{figure}
    \centering
    \subfigure[Fast circle trajectory.]{
        \includegraphics[width=0.21\textwidth]{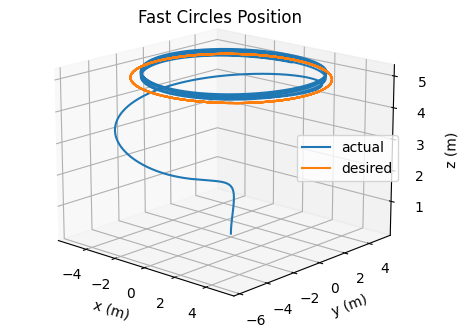}
        \label{fig:fast_circles_3d}
    }
    \hfill
    \subfigure[Flipping loops trajectory.]{
        \includegraphics[width=0.21\textwidth]{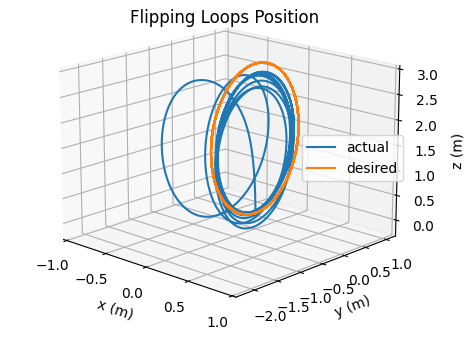}
        \label{fig:flipping_loops_3d}
    }
    \caption{3D position plots of the simulated trajectories. The $z$-axis of both plots is flipped for visual clarity.}
\end{figure}

The results are shown in Figures \ref{fig:fast_circles} and \ref{fig:fast_circles_3d}. The quadrotor converges to the correct altitude within 3 seconds and follows the trajectory fairly well. Due to the aggressiveness of the trajectory and because of the decoupling between the position and rotation controllers, it never quite reaches the correct diameter, but stays fairly close to it. Note that the roll angle throughout the trajectory is around 70 degrees, indicating that the trajectory is quite aggressive.

\subsection{Flipping Loops}

The second trajectory is also sinusoidal. The quadrotor was commanded to do vertical loops in the $yz$-plane, with a $y$ amplitude of 1 m, a $z$ amplitude of 1.5 m, a vertical offset of 1.5 m, and a period of 1.4 s. The commanded heading was zero for the entirety of the trajectory. The trajectory is so fast that the only way the quadrotor can follow it is to flip upside-down whenever it reaches the top of the loop in order to accelerate downward faster than gravity.

\begin{figure}
    \centering
    \includegraphics[width=0.47\textwidth]{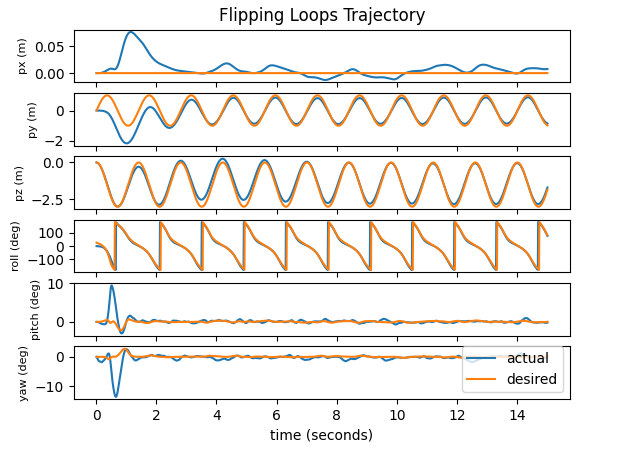}
    \caption{Flipping loop trajectory performance.}
    \label{fig:flip_loops}
\end{figure}

The results are shown in Figures \ref{fig:flip_loops} and \ref{fig:flipping_loops_3d}. The quadrotor converges to the trajectory fairly quickly and follows it well. Note that the roll angle continually exceeds 180 degrees, showing that the quadrotor was indeed flipping upside-down.

\subsection{Upside-down Recovery}

For this simulation, the quadrotor was given an initial roll angle of exactly 180 degrees and commanded to hover in place at $\mathbf{p}_d = \mathbf{0}$. The goal of this trajectory is to verify whether the proposed control scheme is indeed globally stable. The performance of our controller was tested against the controllers presented in \cite{Lee2010} and \cite{Lee2015}. The controller of \cite{Lee2012} was not compared because their choice of error rotation is not defined when $\phi = \pm\pi$. We set $\mathbf{K}_\omega = 30 \mathbf{J}$, and, after a great deal of tuning to ensure good performance, set the parameters of \cite{Lee2015} to (using their notation) $k_1 = 20$, $k_2 = 15$, $\alpha = 1.99$, $\beta = 0.98$, $\delta = 0.05$, $B_{e_\Omega} = 5$, and $k_\omega = 20 \mathbf{J}$. The results are shown in Figure \ref{fig:upside_down}, with \cite{Lee2010} denoted as ``1" and \cite{Lee2015} denoted as ``2". Ours and \cite{Lee2015} were both able to recover the quadrotor, while \cite{Lee2010} was not, for the reason depicted in Figure \ref{fig:err_func}. Our controller was able to flip the quadrotor over more quickly than \cite{Lee2015}, and as a result reached the origin sooner. We believe that this is because the hybrid control scheme of \cite{Lee2015} introduced a non-smooth response when the control configuration jumped at about 0.5 s (this can be seen by the bump in the roll plot), sending the vehicle further from the origin before beginning to converge.

\begin{figure}
    \centering
    \includegraphics[width=0.47\textwidth]{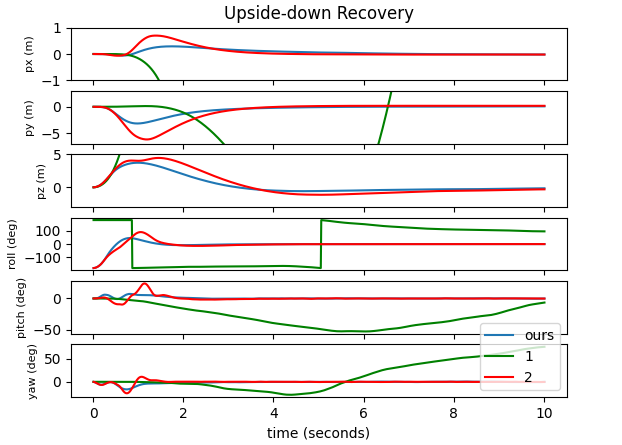}
    \caption{Upside-down recovery performance, comparing to the controllers presented in \cite{Lee2010} and \cite{Lee2015}.}
    \label{fig:upside_down}
\end{figure}

\section{Hardware Experiments}

\subsection{Modified Hardware Controller}

Most quadrotor hardware platforms have an onboard embedded flight control unit (FCU) that tracks attitude or angular rate commands at very high rates (e.g. 1000 Hz) using an inertial measurement unit (IMU). In order to better interface with the FCU, we modified the controller presented in section \ref{sec:geom_ctrl}. We assume that the FCU can achieve a commanded angular rate nearly instantaneously and neglect the angular rate dynamics \eqref{eq:ang_dyn}. The rotational dynamics thereby become $\dot{\mathbf{R}}_b^i = \mathbf{R}_b^i {\boldsymbol{\omega}_c^b}^\wedge$, where $\boldsymbol{\omega}_c^b$ is the angular velocity command sent to the FCU.

Define $\tilde{\mathbf{r}}$ and $\tilde{\boldsymbol{\omega}}$ as in (\ref{eq:r_err},~\ref{eq:om_err}), but replace $\boldsymbol{\omega}_{b/i}^b$ with $\boldsymbol{\omega}_c^b$.
\begin{theorem}
    The control law
    \begin{equation}
        \boldsymbol{\omega}_c^b = \mathbf{R}_d^b \boldsymbol{\omega}_{d/i}^d + J_l(\tilde{\mathbf{r}}) \mathbf{K}_r \tilde{\mathbf{r}}
        \label{eq:om_c}
    \end{equation}
    exponentially drives the error dynamics \eqref{eq:r_err_dot} to zero for any initial value of $\tilde{\mathbf{r}}$ (assuming angular velocity is achieved instantaneously).
\end{theorem}
\begin{proof}
    First, assume that $\tilde{\mathbf{r}}$ is initially in the set \eqref{eq:set_def}. Choose the positive definite Lyapunov function candidate
    \begin{equation}
        \mathcal{V}(\tilde{\mathbf{r}}) = \frac{1}{2} \tilde{\mathbf{r}}^\top \tilde{\mathbf{r}}.
    \end{equation}
    Taking the time derivative, we get
    \begin{equation}
    \begin{split}
        \dot{\mathcal{V}} &= \dot{\tilde{\mathbf{r}}}^\top \tilde{\mathbf{r}} = \tilde{\boldsymbol{\omega}}^\top J_l(\tilde{\mathbf{r}})^{-\top} \tilde{\mathbf{r}} \\
        &= \left( \mathbf{R}_d^b \boldsymbol{\omega}_{d/i}^d - \boldsymbol{\omega}_c^b \right)^\top J_l(\tilde{\mathbf{r}})^{-\top} \tilde{\mathbf{r}}.
        \label{eq:v_dot2}
    \end{split}
    \end{equation}
    Using \eqref{eq:om_c} we get $\dot{\mathcal{V}} = -\tilde{\mathbf{r}}^\top \mathbf{K}_r \tilde{\mathbf{r}}$, which is negative definite and can be bounded above by an exponential function of $\tilde{\mathbf{r}}$, thus the system is exponentially stable on $S$.
    
    We can use a similar argument as the one used in Theorem \ref{thm:full_dyn_ctrl} to show that the set $\bar{S} \setminus S$ is not invariant to the error dynamics \eqref{eq:r_err_dot}. Assume that $\tilde{\mathbf{r}}(0) \in (\bar{S} \setminus S)$. Using \eqref{eq:om_c}, the new error dynamics are
    \begin{equation}
        \dot{\tilde{\mathbf{r}}} = \mp\pi\mathbf{K}_r \mathbf{u},
    \end{equation}
    which, if the set $\bar{S} \setminus S$ is invariant, implies that
    \begin{equation}
        \dot{\mathbf{u}} = -\mathbf{K}_r \mathbf{u}.
        \label{eq:u_dot}
    \end{equation}
    Multiplying both sides of \eqref{eq:u_dot} by $\mathbf{u}^\top$, we get $-\mathbf{u}^\top \mathbf{K}_r \mathbf{u} = 0$, again using the fact that $\mathbf{u}^\top \dot{\mathbf{u}} = 0$. This would imply that $\mathbf{u} = 0$ because $\mathbf{K}_r$ is positive definite, but this is contradictory because $\mathbf{u}$ must be a unit vector. Thus the set $\bar{S} \setminus S$ is not invariant to the dynamics \eqref{eq:r_err_dot}, and the closed loop system is globally exponentially stable.
\end{proof}

Note the use of the left Jacobian in Equation~\eqref{eq:om_c} versus the inverse left Jacobian in Equation~\eqref{eq:tau}. This is because the error rotation term in~\eqref{eq:v_dot} is cancelled by subtraction, whereas in~\eqref{eq:v_dot2} the Jacobian is cancelled by multiplying by the inverse.

\subsection{Results}

We tested the ability of our control scheme to track aggressive trajectories with a quadrotor hardware platform. Our platform uses ROSFlight\footnote{\url{rosflight.org}} as its onboard FCU. To estimate the state of the quadrotor, we flew the vehicle in a room set up with an Optitrack\footnote{\url{optitrack.com}} motion capture system. The control parameters we used were $\mathbf{W}_e = \text{diag}(0.5, 0.5, 0.5, 0.2, 0.2, 0.2, 0.1, 0.1, 0.1)$, $\mathbf{W}_f = \text{diag}(0.1, 0.1, 0.3)$, and $\mathbf{K}_r = \text{diag}(5, 5, 5)$. 

We tested two aggressive trajectories. The first was a 1.4 m diameter circle with a period of 2.2 s, and with a commanded heading of 0. The results are shown in Figure \ref{fig:hw_circles}. After the trajectory time started at about 32 seconds, the quadrotor quickly converged to the trajectory and was able to track it well throughout the run. The oscillatory error in height is likely caused by inaccurate modeling of the mixing matrix $\mathbf{M}$.

The second trajectory was a hand-designed 5th degree B-spline that started and ended in the same position with no initial or terminal velocity and acceleration. The trajectory was a large loop in the $yz$-plane whose required acceleration at the top of the loop is so high that the quadrotor must point its rotors downwards by doing a flip in order to track it. Figure \ref{fig:flip_timelapse} shows a time lapse of the trajectory and Figure \ref{fig:hw_flip} shows the results. The quadrotor tracked the $y$ and $z$ position fairly well through most of the trajectory, and was completely upside-down just before 30 seconds. After it completed the majority of the maneuver it deviated from the commanded trajectory for a moment. This effect could likely be reduced by further refining the trajectory to ensure dynamic feasibility and/or tuning of the control parameters.

\begin{figure}
    \centering
    \includegraphics[width=0.47\textwidth]{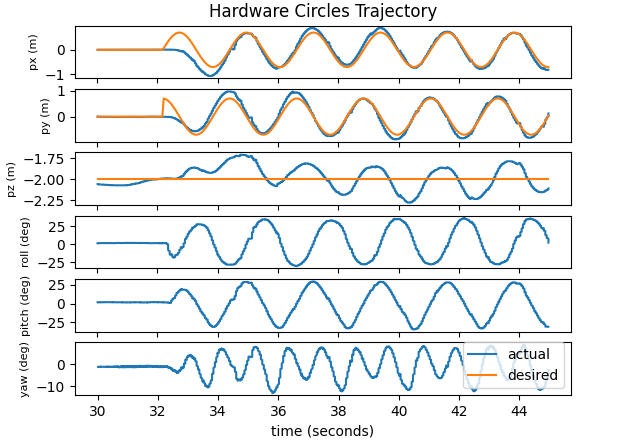}
    \caption{Hardware circle trajectory performance.}
    \label{fig:hw_circles}
\end{figure}

\begin{figure}
    \centering
    \includegraphics[width=0.47\textwidth]{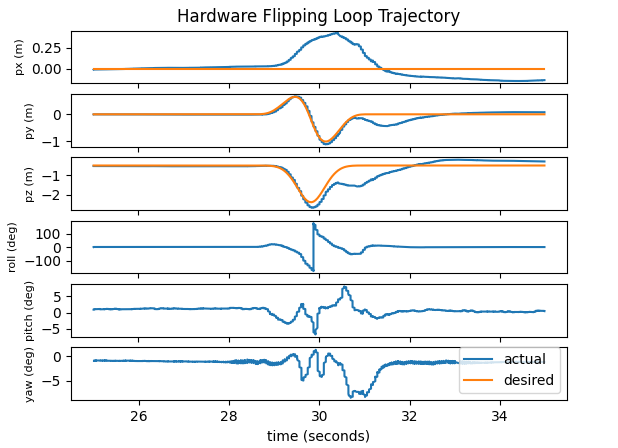}
    \caption{Hardware flipping loop trajectory performance.}
    \label{fig:hw_flip}
\end{figure}

\section{Conclusion}

We have developed a new quadrotor control scheme that is capable of tracking highly aggressive trajectories. Our geometric controller uses the logarithmic map to express rotational error in the Lie algebra of SO(3), which allows us to treat the manifold in a more effective and meaningful manner. We have shown that the proposed geometric controller to be globally attractive, without requiring a complicated hybrid control scheme. Additionally, we have presented an adaptation to this controller that allows it to interface with off-the-shelf quadrotor FCUs and have shown the ability of this control scheme to track highly aggressive trajectories in both simulation and hardware experiments. 


\bibliography{refs}
\bibliographystyle{IEEEtran}

\end{document}